\documentclass[11pt]{article}

\usepackage[preprint]{acl}
\usepackage{times}
\usepackage{latexsym}
\usepackage[T1]{fontenc}
\usepackage[utf8]{inputenc}
\usepackage{microtype}
\usepackage{inconsolata}
\usepackage{graphicx}
\usepackage{etoolbox}
\usepackage{amsmath,amssymb,amsthm}
\usepackage{mathtools}
\usepackage{bm}
\usepackage{xcolor}
\usepackage{hyperref}
\usepackage{fontawesome5}

\definecolor{JevaIconRed}{HTML}{B2483F}
\definecolor{JevaIconTeal}{HTML}{397E79}
\definecolor{JevaResourceRed}{HTML}{7F1D1D}
\usepackage{booktabs,multirow,colortbl}

\definecolor{JevaInk}{HTML}{111111}
\definecolor{JevaTeal}{HTML}{CC0000}
\definecolor{JevaTableHeader}{HTML}{FFF1CC}
\definecolor{JevaRowColor}{HTML}{FBECEC}
\definecolor{JevaRankOneColor}{HTML}{FFBF2B}
\definecolor{JevaRankTwoColor}{HTML}{FFE5A0}
\definecolor{JevaRankThreeColor}{HTML}{FFF1CC}
\newcommand{\JevaModelKey}[1]{\textcolor[HTML]{#1}{\rule{0.75ex}{0.75ex}}\,}
\providecommand{\JevaRankOne}[1]{\cellcolor{JevaRankOneColor}\textbf{#1}}
\providecommand{\JevaRankTwo}[1]{\cellcolor{JevaRankTwoColor}#1}
\providecommand{\JevaRankThree}[1]{\cellcolor{JevaRankThreeColor}#1}

\newsavebox{\JevaTableBox}
\newcommand{\JevaFitTable}[1]{%
  \sbox{\JevaTableBox}{#1}%
  \ifdim\wd\JevaTableBox>\linewidth
    \resizebox{\linewidth}{!}{\usebox{\JevaTableBox}}%
  \else\usebox{\JevaTableBox}\fi}

\newtheorem{proposition}{Proposition}

\DeclareMathOperator{\softmax}{softmax}

\title{Jeva: System One Models from First Principles}

\date{}
\author{ Qi Li \quad Xinchao Wang \\
  National University of Singapore \\
  \texttt{liqi@u.nus.edu; xinchao@nus.edu.sg} \\
}

\begin{document}

\makeatletter
\patchcmd{\@maketitle}
  {\vskip 0.2in plus 1fil minus 0.1in}
  {\vskip 4mm}
  {}
  {\PackageError{jeva}{Title-author spacing patch failed}{}}
\makeatother

\makeatletter
\apptocmd{\@maketitle}{%
  \par
  \noindent

\par\vspace{-25mm}
\noindent
\begin{minipage}{\textwidth}
  \centering
  \begingroup
  \hypersetup{
    colorlinks=true,
    urlcolor=JevaResourceRed
  }

\href{https://huggingface.co/LIQIIIII/Jeva-4B}{%
  \raisebox{-0.18em}{%
    \includegraphics[height=1.5em]{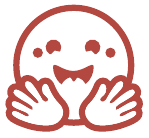}%
  }%
}
\hspace{1.8em}
\href{https://github.com/LiQiiiii/Jeva}{%
  \textcolor{JevaIconRed}{\scalebox{1.5}{\faGithub}}%
}
\hspace{1.8em}
\href{https://liqiiiii.github.io/Jeva/}{%
  \textcolor{JevaIconRed}{\scalebox{1.5}{\faGlobe}}%
}

  \endgroup
\end{minipage}
\par\vspace{3mm}

  \begin{minipage}{\textwidth}
    \centering
    \includegraphics[
      width=0.94\textwidth,
    ]{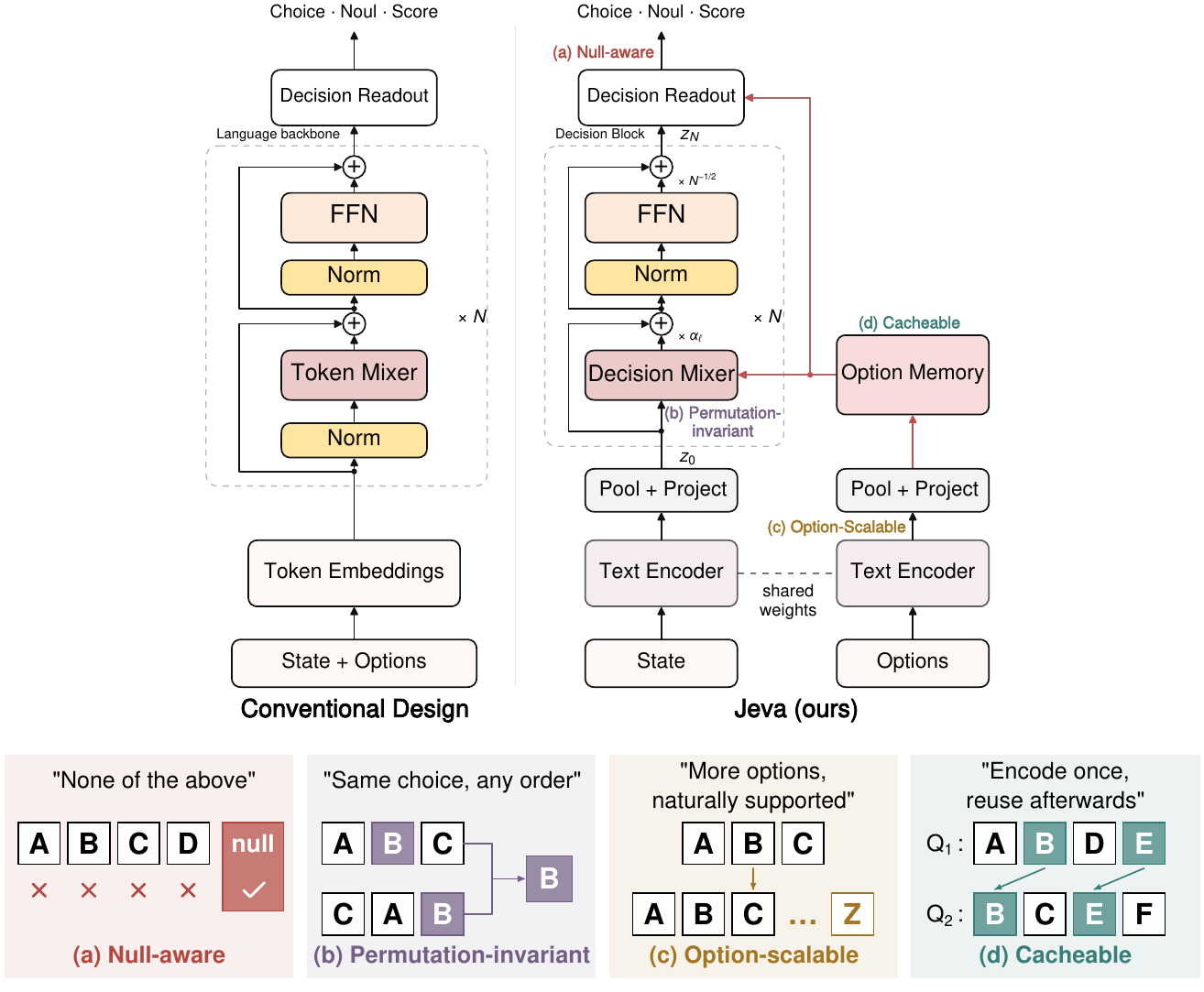}
    \vspace{-3mm}
    \captionof{figure}{%
      \textbf{Top:} a representative joint-encoding system one architecture and Jeva; $N$ denotes the number of blocks in each architecture. \textbf{Bottom:} schematic illustrations of four structural properties: (a) an explicit null outcome when no candidate applies; (b) unchanged choices under candidate reordering; (c) support for variable-size candidate sets; and (d) reuse of option representations across requests.
    }
    \label{fig:overview}
    \vspace{-1mm}
  \end{minipage}
  \vspace{5mm}
  \par\noindent
  \begin{minipage}{\textwidth}
    \begin{abstract}
System one models like Jev enable direct, low-latency decision-making. Yet recent open-source practice of predicting an answer from serialized text does not, by itself, respect the structure of a decision problem. We argue that current designs underemphasize four important structural requirements.
\textbf{(i) null-aware:} Models should distinguish answer absence from uncertainty among candidates;
\textbf{(ii) permutation-invariant:} Reordering alternatives should not change the decision;
\textbf{(iii) option-scalable:} Growing candidate sets should not require redefining the output space;
\textbf{(iv) cacheable:} Request-independent computation over recurring alternatives should be reusable across requests.
We introduce \textbf{Jeva}, a decision-native architecture that addresses these requirements jointly. A shared encoder independently represents the request and each option, while a compact decision stack refines only the request state against immutable option memory. Symmetric aggregation yields permutation-equivariant probabilities; shared scoring accommodates variable-size candidate sets. A separate coverage estimate distinguishes answer absence from candidate competition, enabling explicit null outcomes. Request-independent option representations and layer projections support reuse across newly assembled sets. These mechanisms unify Choice, Noul, and Score outputs.
Across evaluations against Jev and open source models, Jeva combines strong typed-decision performance with near-perfect null detection and no decision flips under any tested permutations. It scales to thousands of candidates and substantially reduces latency through cross-request reuse without changing decisions.
    \end{abstract}
  \end{minipage}
  \par
}{}{%
  \PackageError{jeva}{Unable to insert the teaser figure}{}
}
\makeatother

\makeatletter
\pretocmd{\@maketitle}{\vspace*{-8mm}}{}{%
  \PackageError{jeva}{Unable to adjust title position}{}
}
\makeatother

\maketitle

\twocolumn[{%
  \begin{minipage}{\textwidth}
    \centering
    \includegraphics[width=\linewidth]{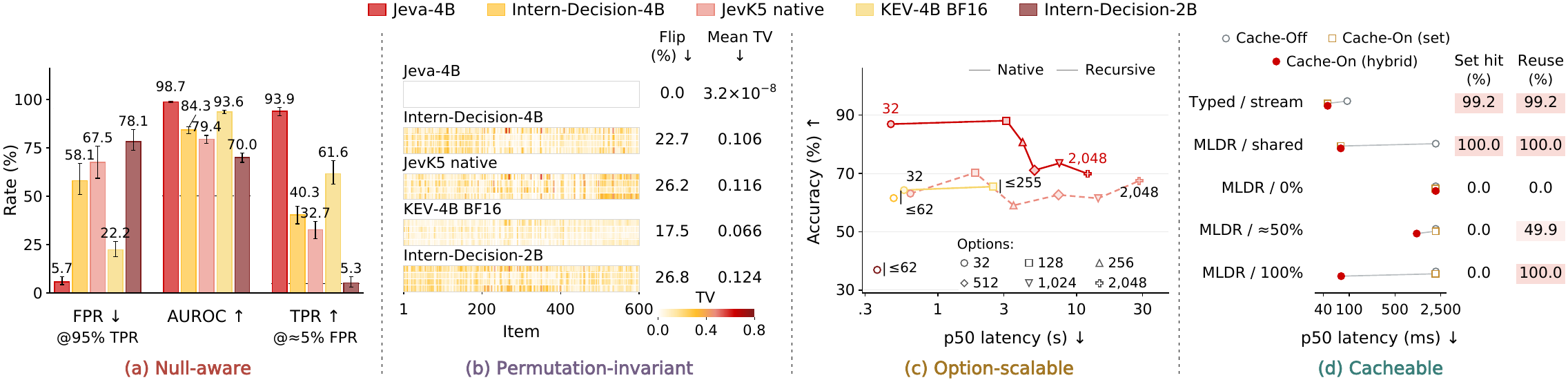}

\captionof{figure}{\textbf{Four structural properties of Jeva.} (a) Null detection on BANKING77~\cite{casanueva2020banking77} with 16 candidates. (b) Candidate-aligned probability changes under Choice permutations from Typed Decisions~\cite{localllama2026typed}, with flip rates and mean total variation (TV). (c) Accuracy--latency trade-offs as candidate counts increase on our ToolRet-based diagnostic~\cite{toolret2025}; solid and dashed lines denote native and recursive selection, respectively. (d) Cache-Off and Cache-On (set-cache, hybrid-cache) latency on Typed Decisions~\cite{localllama2026typed} and MLDR~\cite{chen2024bgem3}, with set-hit rates and effective reuse.}
    \label{fig:structural-properties}
  \end{minipage}
  \par\vspace{2mm}
}]

\section{Introduction}
\label{sec:introduction}
\vspace{-1mm}
Many language-model requests require a compact decision: choose a tool, identify a document, test a condition, or assign a score~\cite{schick2023toolformer,qin2024large,saeidi2018interpretation,zheng2023judging}. System one models like Jev~\citep{almeida2026jev} address these requests directly through typed outputs rather than generated rationales. Such tasks impose structural requirements on the decision process. In candidate selection, reordering the alternatives should preserve the decision, and varying the candidate count should not require a new output space. The model must also recognize when none of the available alternatives applies. Meanwhile, a tool registry, document collection, or label vocabulary can persist across changing requests, creating opportunities to reuse computation and reduce decision latency. Together, these requirements call for a decision-native architecture.

However, many open-source system one models jointly encode requests and candidates as text, then read decisions from answer-token logits or a task-specific prediction head~\citep{jevk5,intern_decision,semif}.
This design places reusable candidate content inside request-specific computation: changing the request or recombining candidates requires processing their joint context again. It also treats an unordered decision set as an ordered sequence, allowing incidental candidate positions to influence predictions. Implementations based on a fixed inventory of answer symbols further restrict candidate counts, while normalizing scores over the supplied candidates does not establish whether any candidate is acceptable. These limitations point to a common architectural issue: semantic representation, candidate interaction, and answer availability are coupled within the same prediction process. A decision-native architecture should separate these roles, preserving context-sensitive judgment while making candidate reuse, order symmetry, and variable-size choice explicit properties of its computation.

We introduce \textbf{Jeva}, a decision-native architecture that separates reusable semantic encoding from set-dependent judgment. In Jeva, as shown in Figure~\ref{fig:overview}, a shared language encoder independently maps the request and each option into compact vectors. A lightweight decision module then refines only the request state by reading an immutable option memory. This lets the state adapt to the candidate set while keeping option representations request-independent. Even a single decision block can reverse the preference between two retained options when the surrounding set changes, capturing interactions that independent candidate scoring cannot express. Symmetric aggregation in the decision mixer and shared candidate scoring in the readout yield permutation-equivariant option probabilities and accommodate variable-size candidate sets without expanding the output parameters. A separate coverage gate allocates probability between null and the candidate set, distinguishing answer absence from uncertainty among alternatives; a conditional distribution resolves competition within the set. Finally, both option vectors and their layer-specific projections can be cached across requests, including newly assembled candidate sets, while state refinement and decision readout remain specific to each request.

We substantiate these design principles through theoretical analysis and empirical comparisons with Jev and other specialized system one models. Our analysis establishes single-block set-dependent expressivity, permutation-equivariant probabilities, and exact reuse of option-side computation, while distinguishing coverage estimation from candidate-score normalization. With a Qwen3.5-4B backbone and one decision block, Jeva combines strong typed-decision performance with nearly perfect null detection and stable choices under all tested candidate permutations. These benefits extend beyond small candidate sets: at 2,048 options, Jeva's native full-set scoring achieves higher accuracy and a speedup of over $2\times$ in median latency over the recursive-selection baseline. Cross-request reuse provides a complementary efficiency gain, delivering a median-latency speedup of over $20\times$ for newly assembled long-document candidate sets without changing decisions. Together, these findings (Figure~\ref{fig:structural-properties}) demonstrate how a decision-native architecture unifies null awareness, permutation invariance, option scalability, and cacheability.
\section{Related Work}
\label{sec:related}

\noindent\textbf{Direct and typed decision models.} Jev provides a unified interface for \textsc{Choice}, \textsc{Noul}, and \textsc{Score} decisions~\citep{almeida2026jev}. Subsequent open-source work explores different readout designs for these typed decisions. SemIf and JevK5 extract option probabilities from answer-token logits without autoregressive decoding~\citep{semif,jevk5}, while Intern-Decision uses decision markers to read distributions over permitted answer symbols~\citep{intern_decision}. ReFlex aggregates predictions across different option orders to reduce sensitivity to candidate position~\citep{reflex}. Beyond token-logit readouts, KEV uses a pointer head over contextualized option positions~\citep{palmer2026kev4b}, and Laya combines a bidirectional language encoder with a dedicated decision head~\citep{laya}. These methods derive decisions from request and candidate token representations. Jeva separates independent text encoding from compact, set-conditioned state updates, making candidate computation reusable while retaining interactions among alternatives.

\noindent\textbf{Independent encoding and late interaction.} Separating query and candidate encoding enables reusable representations and efficient matching. Dense Passage Retrieval encodes questions and passages independently, allowing corpus representations to be precomputed~\citep{karpukhin2020dpr}. ColBERT retains token-level representations to support fine-grained late interaction~\citep{khattab2020colbert}. Contrastive Language Models extend independent encoding and candidate caching to typed Choice, Noul, and Score decisions~\citep{clm}. Jeva refines the request state against the full candidate set before scoring and coverage estimation, enabling set-dependent preferences while keeping candidate vectors and layer projections reusable.

\begin{figure}[t]
  \centering
  \includegraphics[width=\linewidth]{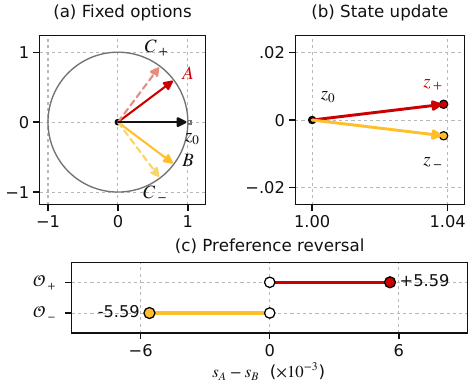}
  \vspace{-8mm}
  \caption{\textbf{Fixed options, set-dependent decisions.}
  Replacing $C_+$ by $C_-$ changes the refined state and reverses
  the $A$--$B$ margin. Hollow markers show the independent-score tie.
  The state zoom uses equal scales; values follow
  Proposition~\ref{prop:jeva-setcontext}.}
  \vspace{-6mm}
  \label{fig:jeva-context-mechanism}
\end{figure}

\noindent\textbf{Set functions and latent attention.} Deep Sets establishes shared transformations and symmetric aggregation as a foundation for permutation-invariant set modeling~\citep{zaheer2017deepsets}. Set Transformer introduces attention-based interactions among set elements~\citep{lee2019settransformer}, while Perceiver processes large inputs through a compact latent representation~\citep{jaegle2021perceiver}. 
Jeva refines a request-conditioned state against immutable option memory, combining set-dependent judgment with reusable representations. Symmetric aggregation and shared scoring yield permutation-equivariant probabilities and support variable-size sets without changing parameter shapes. At fixed width and depth, refinement and readout scale linearly with candidate count over prepared memory.

\noindent\textbf{Null awareness and calibration.} Answerability and abstention address related but distinct questions. SQuAD~2.0 evaluates answerability~\citep{rajpurkar2018squad2}, while selective classification permits abstention to control prediction risk~\citep{geifman2017selective}. Jeva defines null as the absence of a valid candidate. Its coverage gate distinguishes the absence of a valid candidate from uncertainty among alternatives. Post-training temperature scaling provides complementary probability calibration~\citep{guo2017calibration}.

\section{Jeva: A Native Decision Architecture}
\label{sec:method}
\begin{figure}[t]
  \centering
  \includegraphics[width=0.9\linewidth]{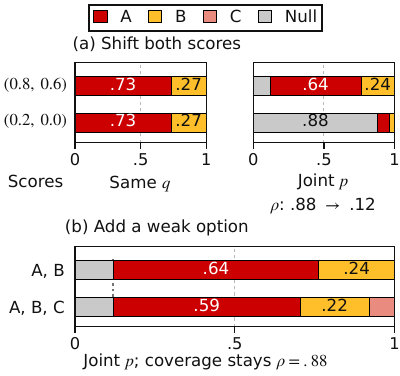}
  \caption{\textbf{Competition differs from coverage.}
  (a) A score shift preserves $q$ but changes $\rho$. (b) Adding $s_C=0.4$ at fixed scores preserves $\rho$ but redistributes $p$. Here $\tau_{\mathrm{dec}}=0.20$, $\tau_{\mathrm{cov}}=0.15$, and $b_{\mathrm{cov}}=0.50$.}
  \label{fig:jeva-coverage-mechanism}
  \vspace{-5mm}
\end{figure}

Jeva independently encodes the request and each option with a shared backbone (Figure~\ref{fig:overview}). A compact decision module refines only the request state against immutable option memory, enabling reuse while retaining set-dependent scoring.

\subsection{Independent Encoding, Collective Decisions}
\label{sec:jeva-encoding}

A request comprises state $x$, question $\iota$, decision type $t$, and candidates $\mathcal O=(o_1,\ldots,o_K)$, with $K\geq2$ in the typed interface. We format inputs with two text templates: $T_x(t,\iota,x)$ combines the decision type, question, and state, while $T_o(t,o_i)$ contains only the decision type and candidate text. Let $E_\theta$ be the language backbone followed by last-nonpadding-token pooling. Bias-free projections give
\begin{equation}
 \begin{aligned}
 z_0&=P_xE_\theta(T_x(t,\iota,x)),\\
 u_i&=\operatorname{nrm}\!\left(P_oE_\theta(T_o(t,o_i))\right),
 \end{aligned}
 \label{eq:jeva-encoding}
\end{equation}
where $\operatorname{nrm}$ is stabilized $\ell_2$ normalization and $z_0,u_i\in\mathbb R^d$. Templates identify each input's type and role; $T_o$ contains no current question, state, candidate index, or other option. Each text is contextualized by the backbone before pooling. The rows of $U=[u_1^\top;\ldots;u_K^\top]$ stay fixed within a forward pass, while the encoder and projections remain trainable.

\noindent\textbf{Why refine the state?}
Let $q$ denote the distribution over candidates conditional on an acceptable candidate being present.  For independent scores $g(x,\iota,o_i)$ normalized by softmax with temperature $\tau>0$, its log odds satisfy
\begin{equation}
 \log\frac{q_i(\mathcal O)}{q_j(\mathcal O)}
 =\frac{g(x,\iota,o_i)-g(x,\iota,o_j)}{\tau}.
 \label{eq:jeva-independent-odds}
\end{equation}
The shared denominator cancels in the ratio: adding or removing other candidates can change probabilities but cannot reverse the preference between two retained options. Jeva instead conditions the scoring state on the candidate set, enabling such reversals through compact per-request computation while keeping option representations reusable.

\begin{figure}[t]
  \centering
  \includegraphics[width=\linewidth]{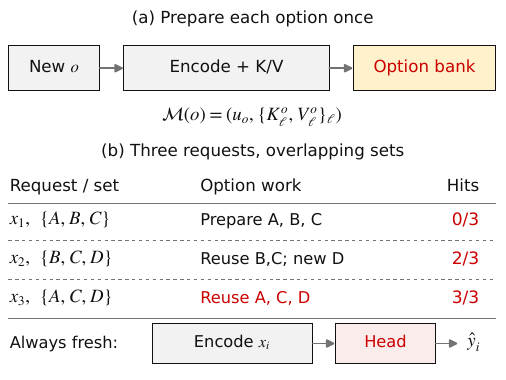}
  \caption{\textbf{Reuse across new requests and sets.}
  With a fixed model, type, and encoding policy, an initially empty
  option bank gives $0/3$, $2/3$, and $3/3$ hits without eviction.
  Every whole-set lookup misses; state encoding and decisions remain fresh.}
  \label{fig:jeva-reuse-mechanism}
  \vspace{-5mm}
\end{figure}

\subsection{Candidate-Conditioned State Refinement}
\label{sec:jeva-mixer}

\begin{table*}[t]
\centering
\fontsize{9}{10.8}\selectfont
\setlength{\tabcolsep}{4.5pt}
\renewcommand{\arraystretch}{1.08}
\setlength{\heavyrulewidth}{0.65pt}
\setlength{\lightrulewidth}{0.4pt}
\JevaFitTable{%
\begin{tabular}{@{}lrrrrrrrr@{}}
\toprule\arrayrulecolor{black}
 & Typed & \multicolumn{2}{c}{BANKING77} & \multicolumn{2}{c}{MASSIVE} & \multicolumn{3}{c}{Mind2Web} \\
\cmidrule(lr){2-2}\cmidrule(lr){3-4}\cmidrule(lr){5-6}\cmidrule(l){7-9}
Model & Test & $K=8$ & $K=16$ & $K=8$ & $K=16$ & Task & Website & Domain \\
\midrule
\JevaModelKey{7F1D1D}Jev-1.13.0 & 74.10 & \JevaRankTwo{86.14} & \JevaRankTwo{82.50} & \JevaRankTwo{84.63} & \JevaRankTwo{81.61} & \JevaRankTwo{66.26} & \JevaRankTwo{59.97} & \JevaRankTwo{62.47} \\
\hline
\JevaModelKey{EA8B7F}JevK5 & 62.30 & 78.18 & 69.03 & 80.16 & 69.70 & 51.47 & 43.30 & 45.60 \\
\JevaModelKey{FFBF2B}InternDecision-4B & \JevaRankOne{78.70} & 80.78 & 75.03 & \JevaRankThree{81.20} & \JevaRankThree{75.79} & 48.02 & 41.17 & 43.61 \\
\JevaModelKey{FFBF2B}InternDecision-2B & \JevaRankThree{75.85} & 48.31 & 35.10 & 45.33 & 41.69 & 29.43 & 27.47 & 26.12 \\
\JevaModelKey{FFBF2B}InternDecision-0.8B & 69.35 & 29.74 & 22.63 & 28.58 & 24.61 & 25.33 & 22.53 & 21.48 \\
\JevaModelKey{7F1D1D}Re\textperiodcentered Flex-4B & 58.85 & 78.77 & 74.38 & 76.63 & 71.39 & \JevaRankThree{52.08} & \JevaRankThree{45.05} & \JevaRankThree{47.64} \\
\JevaModelKey{7F1D1D}SemiF-4B & 62.75 & 69.97 & 67.18 & 69.70 & 67.38 & 48.60$^\dagger$ & 41.16$^\dagger$ & 43.67$^\dagger$ \\
\JevaModelKey{F6D56A}KEV-4B (BF16) & 66.00 & \JevaRankThree{83.47} & \JevaRankThree{82.47} & 76.60 & 73.54 & 40.32 & 34.17 & 35.62 \\
\JevaModelKey{F6D56A}KEV-4B (FP32) & 66.05 & 83.44 & 82.34 & 76.53 & 73.64 & 40.53 & 34.25 & 35.81 \\
\JevaModelKey{7F1D1D}CLM-8B & 43.05 & 26.95 & 18.15 & 33.36 & 23.23 & 21.73 & 21.16 & 19.79 \\
\JevaModelKey{7F1D1D}Valen-Preview-2B & 37.75 & 19.61 & 19.09 & 19.60 & 21.72 & 13.27 & 14.16 & 14.38 \\
\JevaModelKey{7F1D1D}Laya-English & 34.50 & 58.05 & 50.49 & 44.69 & 37.63 & 15.96 & 11.34 & 15.26 \\
\JevaModelKey{7F1D1D}Qwen3.5-4B label LP & 50.85 & 49.84 & 43.12 & 47.31 & 41.63 & 18.14 & 15.30 & 18.34 \\
\JevaModelKey{7F1D1D}Qwen3.5-2B label LP & 41.70 & 48.25 & 35.65 & 30.93 & 25.25 & 12.46 & 10.73 & 11.47 \\
\midrule
\rowcolor{JevaRowColor}
\JevaModelKey{CC0000}\textbf{Jeva-4B} & \JevaRankTwo{78.30} & \JevaRankOne{92.05} & \JevaRankOne{91.62} & \JevaRankOne{91.09} & \JevaRankOne{90.22} & \JevaRankOne{79.99} & \JevaRankOne{62.25} & \JevaRankOne{66.23} \\
\bottomrule\arrayrulecolor{black}
\end{tabular}}
\par\smallskip
\vspace{-4mm}
\caption{\textbf{First-argmax accuracy (\%, higher is better), including null outcomes where applicable.} Mind2Web uses $K=16$. $^\dagger$Supported-only accuracy; coverage is 99.2\%, 99.8\%, and 97.3\% respectively; all other entries have full coverage. Dark-to-light gold marks the top three distinct full-coverage scores per column.}
\label{tab:general-accuracy}
\vspace{-5mm}
\end{table*}
Each of $L$ blocks reads $U$ through $H$ heads. For head $h$, define $q_\ell^h=\operatorname{nrm}([W_\ell^Qz_\ell]_h)$, $k_{\ell i}^h=\operatorname{nrm}([W_\ell^Ku_i]_h)$, and $v_{\ell i}^h=[W_\ell^Vu_i]_h$. The candidate-weighted message is
\begin{equation}
 \begin{aligned}
 a_{\ell i}^h
 &=\frac{\exp((q_\ell^h)^\top k_{\ell i}^h/\tau_\ell)}
 {\sum_j\exp((q_\ell^h)^\top k_{\ell j}^h/\tau_\ell)},\\
 m_\ell&=W_\ell^O\operatorname{Concat}_{h=1}^{H}
 \left(\sum_i a_{\ell i}^h v_{\ell i}^h\right).
 \end{aligned}
 \label{eq:jeva-mixer}
\end{equation}
Padding is masked. For compatibility logits $c_i=(q_\ell^h)^\top k_{\ell i}^h$, these weights uniquely maximize $a^\top c+\tau_\ell\mathcal H(a)$ over the probability simplex, yielding an entropy-regularized allocation of candidate evidence for state refinement. A derivation is provided in Appendix~\ref{app:jeva-entropy}. The state undergoes two residual updates:
\begin{equation}
 \begin{aligned}
 \widetilde z_\ell&=z_\ell+\alpha_\ell m_\ell,\\
 z_{\ell+1}&=\widetilde z_\ell+
 L^{-1/2}F_\ell(\operatorname{RMSNorm}(\widetilde z_\ell)),
 \end{aligned}
 \label{eq:jeva-update}
\end{equation}
where $F_\ell$ is a bias-free SwiGLU network and $\alpha_\ell=\tau_\ell\sigma(\gamma_\ell)/\sqrt L$. This coupling ensures $0<\alpha_\ell/\tau_\ell<1/\sqrt{L}$. All temperatures are learned positive scalars. Only $z_\ell$ evolves; keys and values depend solely on the unchanged options.

\begin{proposition}[Set-dependent preference with one block]
\label{prop:jeva-setcontext}
There exist fixed Jeva parameters and two candidate sets containing the same options $A,B$ such that, for the same request, a single block uniquely selects $A$ in one set and $B$ in the other. No independent scorer in Eq.~\eqref{eq:jeva-independent-odds} with set-independent tie-breaking can realize both choices.
\end{proposition}
\begin{proof}
Use one two-dimensional head, identity projections, and zero FFN output. Score each candidate by $s_i=z_1^\top u_i$. Let $z_0=(1,0)$, $A=(4,3)/5$, $B=(4,-3)/5$, and $C_\pm=(3,\pm4)/5$. Compare sets $(A,B,C_+)$ and $(A,B,C_-)$. All vectors are unit length. With $\tau=3/20$ and $\gamma=-\log2$, the residual scale is $\alpha=1/20$; the third option receives weight $w=(1+2e^{4/3})^{-1}$. The score differences are
\begin{equation}
 s_A(\mathcal O_\pm)-s_B(\mathcal O_\pm)
 =\pm24\alpha w/25.
 \label{eq:jeva-context-main}
\end{equation}
The winning $A$ or $B$ also exceeds $C_\pm$: its initial advantage $1/5$ can shrink by at most $2\alpha<1/5$. The winners are therefore unique, whereas an independent scorer cannot reverse their ordering by Eq.~\eqref{eq:jeva-independent-odds}.
\end{proof}
Figure~\ref{fig:jeva-context-mechanism} shows this construction; Appendix~\ref{app:jeva-context} embeds it in the deployed width and head count. Thus reusable features do not imply independent decisions. The mixer forms the message directly, without storing a $K\times K$ interaction matrix.

\subsection{Competition, Coverage, and Typed Outputs}
\label{sec:jeva-readout}

\begin{table*}[t]
\centering
\fontsize{9}{10.8}\selectfont
\setlength{\tabcolsep}{4.5pt}
\renewcommand{\arraystretch}{1.08}
\setlength{\heavyrulewidth}{0.65pt}
\setlength{\lightrulewidth}{0.4pt}
\arrayrulecolor{black}

\JevaFitTable{%
\begin{tabular}{@{}lrrrrrrr@{}}
\toprule
 & \multicolumn{2}{c}{Noul $\downarrow$}
 & \multicolumn{2}{c}{Score $\downarrow$}
 & \multicolumn{3}{c}{Choice permutations} \\
\cmidrule(lr){2-3}\cmidrule(lr){4-5}\cmidrule(l){6-8}

Model & Brier & MAE & MAE & RMSE
& Flip (\%) $\downarrow$
& All-order (\%) $\uparrow$
& Mean TV $\downarrow$ \\
\midrule

\JevaModelKey{7F1D1D}Jev-1.13.0
& \JevaRankThree{0.0342}
& \JevaRankThree{0.1456}
& 0.3744 & 0.4795
& \JevaRankTwo{17.00}
& \JevaRankTwo{66.67}
& 0.0858 \\
\hline

\JevaModelKey{EA8B7F}JevK5
& 0.0987 & 0.2448 & 0.4370 & 0.5935
& 26.17 & 47.83 & 0.116 \\

\JevaModelKey{FFBF2B}InternDecision-4B
& \JevaRankTwo{0.0340}
& \JevaRankTwo{0.1431}
& \JevaRankTwo{0.2488}
& \JevaRankOne{0.3208}
& 22.67 & \JevaRankThree{64.50} & 0.106 \\

\JevaModelKey{FFBF2B}InternDecision-2B
& 0.0440
& 0.1649
& \JevaRankThree{0.2641}
& \JevaRankThree{0.3433}
& 26.83 & 60.33 & 0.124 \\

\JevaModelKey{FFBF2B}InternDecision-0.8B
& 0.0503 & 0.1835 & 0.3145 & 0.4013
& 45.00 & 45.33 & 0.135 \\

\JevaModelKey{7F1D1D}Re\textperiodcentered Flex-4B
& 0.1768 & 0.3464 & 0.4795 & 0.6490
& 26.50 & 54.67 & 0.116 \\

\JevaModelKey{7F1D1D}SemiF-4B
& 0.0913 & 0.2439 & 0.4722 & 0.6134
& 32.33 & 50.33 & 0.165 \\

\JevaModelKey{F6D56A}KEV-4B (BF16)
& 0.0888 & 0.2282 & 0.3987 & 0.5669
& 17.50
& 55.50 & \JevaRankThree{0.0658} \\

\JevaModelKey{F6D56A}KEV-4B (FP32)
& 0.0887 & 0.2283 & 0.3989 & 0.5675
& \JevaRankThree{17.17}
& 55.83 & \JevaRankTwo{0.0657} \\

\JevaModelKey{7F1D1D}Valen-Preview-2B
& 0.1509 & 0.3348 & 0.7786 & 0.9274
& 31.17 & 24.17 & 0.139 \\

\JevaModelKey{7F1D1D}Laya-English
& 0.1514 & 0.3206 & 0.7083 & 0.8383
& 66.83 & 11.83 & 0.139 \\

\JevaModelKey{7F1D1D}Qwen3.5-4B label LP
& 0.1623 & 0.3213 & 0.8095 & 1.0235
& 84.50 & 14.17 & 0.414 \\

\JevaModelKey{7F1D1D}Qwen3.5-2B label LP
& 0.1411 & 0.3224 & 0.6221 & 0.7730
& 76.83 & 18.17 & 0.199 \\

\midrule
\rowcolor{JevaRowColor}
\JevaModelKey{CC0000}\textbf{Jeva-4B}
& \JevaRankOne{0.0184}
& \JevaRankOne{0.0939}
& \JevaRankOne{0.2315}
& \JevaRankTwo{0.3222}
& \JevaRankOne{0.00}
& \JevaRankOne{76.67}
& \JevaRankOne{3.17e-08} \\
\bottomrule\arrayrulecolor{black}
\end{tabular}%
}
\vspace{-2mm}
\caption{
\textbf{Typed fidelity and order stability.} Flip measures any decision change; all-order accuracy requires correctness under every tested order. Mean TV compares candidate-aligned probability distributions.}
\label{tab:typed-stability}
\vspace{-5mm}
\end{table*}
Jeva computes final scores $s_i=z_L^\top u_i$ and conditional candidate probabilities $q=\softmax(s/\tau_{\mathrm{dec}})$. Neither the initial nor the final state is unit-normalized, retaining magnitude information for confidence estimation. A common score shift leaves $q$ unchanged: relative preference alone cannot establish whether any candidate is acceptable. To form a full distribution $p$ over candidates and null ($\bot$), Jeva separately estimates the probability $\rho$ that at least one acceptable candidate is present:
\begin{equation}
 \begin{gathered}
 \rho=\sigma\!\left(\frac{\max_i s_i-b_{\mathrm{cov}}}
 {\tau_{\mathrm{cov}}}\right),\\
 p_\bot=1-\rho,\qquad p_i=\rho q_i.
 \end{gathered}
 \label{eq:jeva-coverage}
\end{equation}
Here, $b_{\mathrm{cov}}$ is a learned coverage threshold and $\tau_{\mathrm{cov}}>0$ is a learned temperature. Writing $Y$ for the decision outcome, this implements $P(Y=i)=P(Y\ne\bot)P(Y=i\mid Y\ne\bot)$. Closed-set requests use $q$; open-set requests use $p$.
\begin{proposition}[Coverage without a cardinality denominator]
\label{prop:jeva-coverage}
At fixed final scores, adding options below the existing maximum leaves
$\rho$ unchanged. In contrast, a flat softmax with null score $b_0$ and
$K$ identical candidate scores $s$ assigns
\begin{equation}
 p_\bot^{\mathrm{flat}}
 =\sigma\!\left((b_0-s)/\tau-\log K\right).
 \label{eq:jeva-null-cardinality-main}
\end{equation}
\end{proposition}
\begin{proof}
Adding scores below the existing maximum leaves $\rho$ unchanged. For flat softmax, dividing numerator and denominator by $e^{b_0/\tau}$ gives $p_\bot^{\mathrm{flat}}=1/(1+K e^{(s-b_0)/\tau})$, equivalent to Eq.~\eqref{eq:jeva-null-cardinality-main}.
\end{proof}
Figure~\ref{fig:jeva-coverage-mechanism} illustrates these effects through controlled changes to the final scores. Set refinement may still change scores, and the joint argmax over $p$ need not remain fixed when candidates are added (Appendix~\ref{app:jeva-coverage-proof}).

The same computation supports \textsc{Choice} through candidate selection, \textsc{Noul} through the positive Boolean candidate's probability $q_{i^+}$, and \textsc{Score} through $\sum_iq_iv_i$ for numeric rubric values $v_i$. Rubric values and the positive designation determine output interpretation, not extra encoder inputs. Separate temperatures govern evidence aggregation, competition, and coverage.

\subsection{Symmetry, Reuse, and Scaling}
\label{sec:jeva-symmetry}

An atomic option memory stores both its semantic vector and projected views:
\begin{equation}
 \mathcal M_t(o_i)=\left(u_i,
 \{k_{\ell i}^{h},v_{\ell i}^{h}\}_{\ell,h}\right).
 \label{eq:jeva-memory}
\end{equation}
Requests gather these entries in their current candidate order. Our set cache stores assembled candidate tuples; the hybrid cache additionally reuses individual options across newly assembled sets. State encoding and decision refinement still execute for every request. Figure~\ref{fig:jeva-reuse-mechanism} illustrates this distinction.

\begin{proposition}[Candidate symmetry and exact reuse]
\label{prop:jeva-symmetry}
With fixed weights and deterministic exact-arithmetic evaluation, permuting candidates and their metadata permutes candidate probabilities but leaves $z_L$, null probability, and typed Noul/Score outputs unchanged. Replacing option computations by identical cached entries also leaves all outputs unchanged.
\end{proposition}
\begin{proof}
Reordering keys and values only reindexes the sum in Eq.~\eqref{eq:jeva-mixer}. Induction through the state updates therefore preserves $z_L$; shared scoring permutes its outputs. Cached entries are identical because Eq.~\eqref{eq:jeva-memory} contains no request or other candidate. Substitution preserves every subsequent operation.
\end{proof}
A unique winning option is invariant as a semantic item. Ties and finite-precision effects are discussed in Appendix~\ref{app:jeva-method}. No order-consistency loss is required. At fixed width $d$ and FFN width $f$, prepared-memory decision cost is $O(L(d^2+Kd+df))$, linear in $K$. A bank of $M$ options stores $O(M(L+1)d)$ values. New options require encoding and projection, but neither additional output parameters nor a joint token window. Each text has its own encoder length budget; this does not remove individual-text length limits. The $\Omega(K)$ cost of emitting a dense distribution already matches the head's order in candidate count.

\begin{figure*}[t]
  \centering
  \includegraphics[width=.96\linewidth]{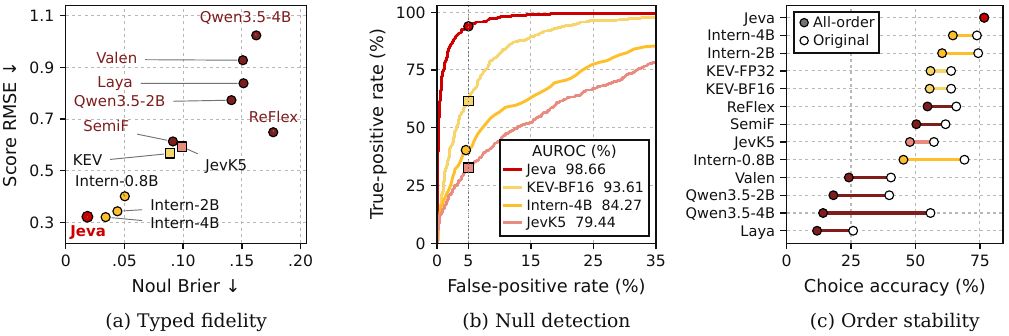}
  \vspace{-2mm}
  \caption{\textbf{Decision quality and stability.} (a) Noul Brier versus Score RMSE;
  lower is better, with different units. (b) Banking77 $K=16$ null ROC:
  the legend gives full AUROC; markers show each measured near-5\% FPR point.
  (c) Original and all-order accuracy on 600 Choice examples; filled markers
  denote all-order accuracy, and segments show permutation sensitivity.}
  \label{fig:decision-quality}
  \vspace{-5mm}
\end{figure*}
\subsection{Direct Supervised Learning}
\label{sec:jeva-training}

For normalized open-set targets $(y_\bot,y_1,\ldots,y_K)$, let $c_y=\sum_i y_i$ and $g=\operatorname{logit}(\rho)$. We minimize
\begin{equation}
 \ell=\operatorname{BCEWithLogits}(g,c_y)-\sum_i y_i\log q_i.
 \label{eq:jeva-loss}
\end{equation}
This primary loss is joint cross-entropy on $p$: all-null examples supervise availability without imposing a spurious candidate ranking. Closed sets use candidate cross-entropy. Stage-dependent Noul and Score auxiliaries, initialization, and optimization details appear in Appendix~\ref{app:jeva-method}.

Jeva-4B uses Qwen3.5-4B with pooled width $2560$, decision width $d=768$, $L=1$, $H=8$, and $f=3072$. We train rank-64 LoRA~\citep{hu2022lora} with both projections and the decision head (13.37M parameters), using separate AdamW parameter groups and global gradient clipping. The curriculum spans typed decisions, intent selection, verification, reading/reasoning, and large-candidate retrieval; stage transitions retain the architecture and preceding weights. LoRA is merged for serving, and memories are prepared for the resulting fixed model.

\section{Experiments}
\label{sec:experiments}

We evaluate decision quality, structural properties, and efficiency using a common suite and controlled permutation, batching, and cache tests. The suite includes Typed Decisions~\citep{localllama2026typed} (2,000 test decisions), BANKING77~\citep{casanueva2020banking77} and MASSIVE~\citep{fitzgerald2023massive} with 8 or 16 candidates, and three Mind2Web~\citep{deng2023mind2web} slices selecting from 16 supplied candidates. Baselines retain their documented interfaces and native runtimes. For open-set evaluation, JevK5 and CLM receive an explicit \texttt{None} candidate, while typed baselines such as Jev, InternDecision, and KEV answer an auxiliary Noul question about candidate availability. For the latter, estimated availability scales the conditional candidate probabilities, with the remaining mass assigned to null. Accuracy uses first-argmax predictions, including null where applicable, on fully supported slices; unsupported requests are reported separately. Local latency uses synchronized H200 timing; within-model comparisons fix the checkpoint and query stream.

\subsection{General and typed decision quality}

Table~\ref{tab:general-accuracy} shows that Jeva leads seven of eight slices in full-coverage accuracy, exceeding 90\% on BANKING77 and MASSIVE at both candidate counts. It achieves 79.99\%, 62.25\%, and 66.23\% on Mind2Web cross-task, cross-website, and cross-domain selection, respectively. On Typed Decisions, Jeva reaches 78.30\%, compared with InternDecision-4B's 78.70\%.

\begin{figure}[tbp]
  \centering
  \includegraphics[width=\linewidth]{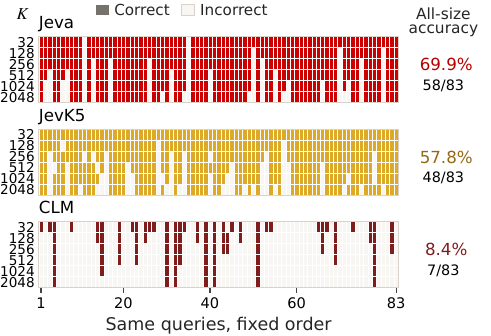}
  \vspace{-7mm}
\caption{\textbf{Correctness across nested candidate sets.} Columns show 83 fixed queries; rows indicate candidate count $K$. Colored cells denote correct predictions. All-size accuracy requires correctness at every tested size.}
\label{fig:nested-candidates}
\vspace{-5mm}
\end{figure}
Table~\ref{tab:typed-stability} evaluates 600 Noul and 800 Score examples. Jeva achieves the lowest Noul Brier (0.0184) and MAE (0.0939) against soft probability targets, and the lowest Score MAE (0.2315). Its Score RMSE is 0.3222, close to InternDecision-4B's 0.3208. Score errors compare predicted expectations with targets in the original rubric units.

We evaluate 600 Choice examples in the original, reversed, and up to three seeded random orders, aligning candidate identities. Jeva records zero flips, mean TV of $3.17\times10^{-8}$, and identical original and all-order accuracy of 76.67\%, consistent with Proposition~\ref{prop:jeva-symmetry}. Jev-1.13.0, InternDecision-4B, and JevK5 exhibit flip rates of 17.00\%, 22.67\%, and 26.17\%, with all-order accuracies of 66.67\%, 64.50\%, and 47.83\%, respectively. All-order accuracy requires correctness under every tested order.
\begin{figure}[tbp]
  \centering
  \vspace{-3mm}
  \includegraphics[width=\linewidth]{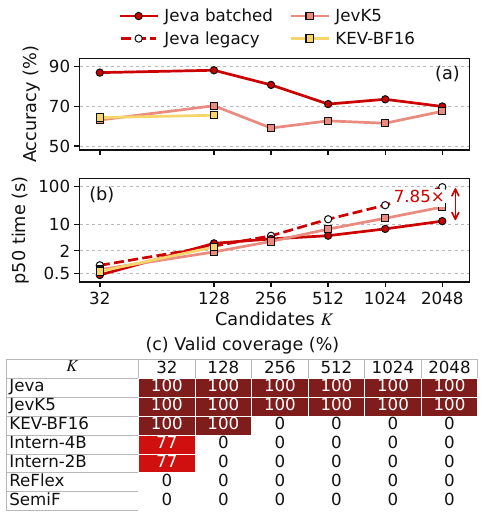}
\caption{\textbf{ToolRet across six disjoint candidate-count cohorts.} (a) Full-support tied-top accuracy. (b) Natural-stream p50 latency; Jeva uses 256-row/131,072-token batching, with legacy encoding dashed. (c) Valid coverage; 77\% denotes 65/84 requests. Unsupported slices are excluded from accuracy.}
  \label{fig:toolret-scaling-full}
  \vspace{-5mm}
\end{figure}
\subsection{Null detection}

On 3,080 matched BANKING77 requests with 16 candidates (588 answer-absent), Jeva achieves 98.66\% null-detection AUROC, versus 93.61\% for KEV-BF16, 84.27\% for InternDecision-4B, and 79.44\% for JevK5. It detects 93.88\% of absent answers at 5.02\% FPR (Figure~\ref{fig:decision-quality}(b)). Baselines use auxiliary coverage Noul outputs or an explicit None candidate. Near-5\% FPR thresholds are each model's empirical 95th percentile of answer-present scores; strictly higher scores predict null.

\begin{figure}[tbp]
  \centering
  \includegraphics[width=\linewidth]{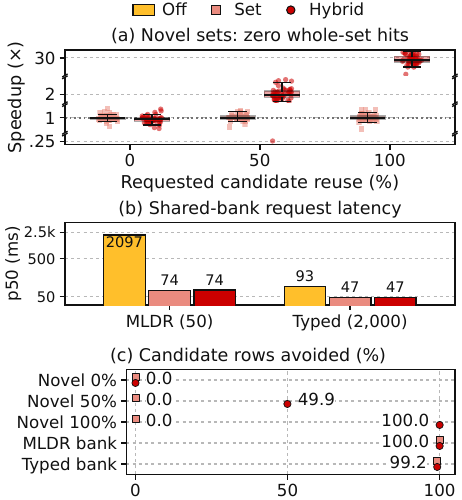}
\caption{\textbf{Candidate reuse} (32-row/16,384-token batching). (a) Paired speedups over Cache-Off: 50 novel-set requests per reuse level, zero set hits, broken axis; boxes show median/IQR with 1.5-IQR whiskers. (b) Shared-bank p50 latency. (c) Encoded rows avoided. Priming is timed separately; no answers are cached.}
  \label{fig:candidate-reuse-full}
  \vspace{-5mm}
\end{figure}

\subsection{Candidate scaling and batching}

Figure~\ref{fig:toolret-scaling-full}(a,c) reports accuracy and coverage on 500 ToolRet requests in distinct-query cohorts with $K\in\{32,128,256,512,1024,2048\}$. Positives are retained (389 multi-positive requests), with weak or unjudged added negatives. Any highest-target-score candidate counts as correct.

Jeva covers all requests with 78.40\% aggregate accuracy, versus 64.00\% for recursive JevK5, while scoring each complete set in one decision-head pass. Both InternDecision variants support only 65/84 requests at $K=32$; larger cohorts exceed their registered limits. Unsupported slices are excluded from accuracy comparisons.

Figure~\ref{fig:toolret-scaling-full}(b) shows the effect of length-aware batching (256 rows/131,072 padded tokens), which preserves cohort accuracies, candidate order, and cache policy. Compared with chunk-2 encoding, p50 latency falls from 31.730 to 7.519~s at $K=1024$ and from 94.459 to 12.040~s at $K=2048$ ($7.85\times$). At $K=2048$, Jeva achieves 69.88\% accuracy, versus JevK5's 67.47\% at 28~s. Encoder batching retains full-set access in the decision head.

\noindent\textbf{Robustness to candidate-set expansion.} Figure~\ref{fig:nested-candidates} evaluates the same 83 queries as nested candidate sets expand from 32 to 2,048 options. Jeva remains correct at every size on 58 queries, versus 48 for JevK5 and 7 for CLM. Its all-size accuracy is 69.88\%, exceeding JevK5 by 12.05 pts.

\subsection{Reuse across states and candidate sets}

Figure~\ref{fig:candidate-reuse-full} compares Cache-Off, exact-set caching (Set), and additional individual-option reuse (Hybrid), using fixed weights, cyclic mode ordering, and cache resets. Only candidate computations are cached, never previous answers.

For 50 matched MLDR queries at 0\%, approximately 50\%, and 100\% option reuse, every novel-set lookup misses the whole-set cache. Hybrid achieves $1.95\times$ and $27.98\times$ speedups in p50 latency at partial and full reuse, while Set remains near 2.1~s. At full reuse, p50 falls from 2,102.0 to 75.1~ms; the median paired speedup is $27.86\times$.

For shared banks, Typed Decisions p50 falls from 93.1~ms to 46.6~ms (Set) and 46.7~ms (Hybrid), avoiding 99.2\% of candidate-row encodings. Primed MLDR p50 falls from 2,097.3~ms to 73.9/74.2~ms. Typed timings include cold starts; MLDR priming is timed separately and excluded. No decisions change across conditions; the maximum absolute probability difference is 0.00736.
\section{Conclusion}
\label{sec:conclusion}

We introduce Jeva, a native architecture specifically designed for system one decision makings. Jeva models decisions as an evolving state reading reusable option memory, connecting permutation symmetry, variable candidate counts, explicit null probability, and caching. The trained Jeva-4B combines broad decision capability with strong null detection, stable predictions under reordering, and savings from hybrid cache. Large-set results show the practical value of candidate-indexed readout. 

\bibliography{references}

@inproceedings{saeidi2018interpretation,
  title        = {Interpretation of Natural Language Rules in Conversational Machine Reading},
  author       = {Saeidi, Marzieh and Bartolo, Max and Lewis, Patrick and Singh, Sameer and Rockt{\"a}schel, Tim and Sheldon, Mike and Bouchard, Guillaume and Riedel, Sebastian},
  booktitle    = {Proceedings of the 2018 Conference on Empirical Methods in Natural Language Processing},
  pages        = {2087--2097},
  year         = {2018},
  doi          = {10.18653/v1/D18-1233},
  url          = {https://aclanthology.org/D18-1233/}
}

@inproceedings{zaheer2017deepsets,
  title        = {Deep Sets},
  author       = {Zaheer, Manzil and Kottur, Satwik and Ravanbakhsh, Siamak and P{\'o}czos, Barnab{\'a}s and Salakhutdinov, Ruslan and Smola, Alexander J.},
  booktitle    = {Advances in Neural Information Processing Systems},
  volume       = {30},
  year         = {2017},
  url          = {https://proceedings.neurips.cc/paper/2017/hash/f22e4747da1aa27e363d86d40ff442fe-Abstract.html}
}

@inproceedings{lee2019settransformer,
  title        = {Set Transformer: A Framework for Attention-based Permutation-Invariant Neural Networks},
  author       = {Lee, Juho and Lee, Yoonho and Kim, Jungtaek and Kosiorek, Adam R. and Choi, Seungjin and Teh, Yee Whye},
  booktitle    = {Proceedings of the 36th International Conference on Machine Learning},
  series       = {Proceedings of Machine Learning Research},
  volume       = {97},
  pages        = {3744--3753},
  year         = {2019},
  url          = {https://proceedings.mlr.press/v97/lee19d.html}
}

@inproceedings{jaegle2021perceiver,
  title        = {Perceiver: General Perception with Iterative Attention},
  author       = {Jaegle, Andrew and Gimeno, Felix and Brock, Andy and Vinyals, Oriol and Zisserman, Andrew and Carreira, Joao},
  booktitle    = {Proceedings of the 38th International Conference on Machine Learning},
  series       = {Proceedings of Machine Learning Research},
  volume       = {139},
  pages        = {4651--4664},
  year         = {2021},
  url          = {https://proceedings.mlr.press/v139/jaegle21a.html}
}

@inproceedings{karpukhin2020dpr,
  title        = {Dense Passage Retrieval for Open-Domain Question Answering},
  author       = {Karpukhin, Vladimir and Oguz, Barlas and Min, Sewon and Lewis, Patrick and Wu, Ledell and Edunov, Sergey and Chen, Danqi and Yih, Wen-tau},
  booktitle    = {Proceedings of the 2020 Conference on Empirical Methods in Natural Language Processing (EMNLP)},
  pages        = {6769--6781},
  year         = {2020},
  doi          = {10.18653/v1/2020.emnlp-main.550},
  url          = {https://aclanthology.org/2020.emnlp-main.550/}
}

@inproceedings{khattab2020colbert,
  title        = {{ColBERT}: Efficient and Effective Passage Search via Contextualized Late Interaction over {BERT}},
  author       = {Khattab, Omar and Zaharia, Matei},
  booktitle    = {Proceedings of the 43rd International ACM SIGIR Conference on Research and Development in Information Retrieval},
  year         = {2020},
  url          = {https://arxiv.org/abs/2004.12832}
}

@inproceedings{hu2022lora,
  title        = {{LoRA}: Low-Rank Adaptation of Large Language Models},
  author       = {Hu, Edward J. and Shen, Yelong and Wallis, Phillip and Allen-Zhu, Zeyuan and Li, Yuanzhi and Wang, Shean and Wang, Lu and Chen, Weizhu},
  booktitle    = {International Conference on Learning Representations},
  year         = {2022},
  url          = {https://openreview.net/forum?id=nZeVKeeFYf9}
}

@inproceedings{guo2017calibration,
  title        = {On Calibration of Modern Neural Networks},
  author       = {Guo, Chuan and Pleiss, Geoff and Sun, Yu and Weinberger, Kilian Q.},
  booktitle    = {Proceedings of the 34th International Conference on Machine Learning},
  series       = {Proceedings of Machine Learning Research},
  volume       = {70},
  pages        = {1321--1330},
  year         = {2017},
  url          = {https://proceedings.mlr.press/v70/guo17a.html}
}

@inproceedings{geifman2017selective,
  title        = {Selective Classification for Deep Neural Networks},
  author       = {Geifman, Yonatan and El-Yaniv, Ran},
  booktitle    = {Advances in Neural Information Processing Systems},
  volume       = {30},
  year         = {2017},
  url          = {https://proceedings.neurips.cc/paper_files/paper/2017/hash/4a8423d5e91fda00bb7e46540e2b0cf1-Abstract.html}
}

@inproceedings{rajpurkar2018squad2,
  title        = {Know What You Don't Know: Unanswerable Questions for {SQuAD}},
  author       = {Rajpurkar, Pranav and Jia, Robin and Liang, Percy},
  booktitle    = {Proceedings of the 56th Annual Meeting of the Association for Computational Linguistics (Volume 2: Short Papers)},
  pages        = {784--789},
  year         = {2018},
  doi          = {10.18653/v1/P18-2124},
  url          = {https://aclanthology.org/P18-2124/}
}

@inproceedings{casanueva2020banking77,
  title        = {Efficient Intent Detection with Dual Sentence Encoders},
  author       = {Casanueva, I{\~n}igo and Tem{\v{c}}inas, Tadas and Gerz, Daniela and Henderson, Matthew and Vuli{\'c}, Ivan},
  booktitle    = {Proceedings of the 2nd Workshop on Natural Language Processing for Conversational AI},
  pages        = {38--45},
  year         = {2020},
  doi          = {10.18653/v1/2020.nlp4convai-1.5},
  url          = {https://aclanthology.org/2020.nlp4convai-1.5/}
}

@inproceedings{deng2023mind2web,
  title        = {{Mind2Web}: Towards a Generalist Agent for the Web},
  author       = {Deng, Xiang and Gu, Yu and Zheng, Boyuan and Chen, Shijie and Stevens, Samuel and Wang, Boshi and Sun, Huan and Su, Yu},
  booktitle    = {Advances in Neural Information Processing Systems},
  volume       = {36},
  year         = {2023},
  url          = {https://proceedings.neurips.cc/paper_files/paper/2023/hash/5950bf290a1570ea401bf98882128160-Abstract-Datasets_and_Benchmarks.html}
}

@inproceedings{toolret2025,
  title        = {Retrieval Models Aren't Tool-Savvy: Benchmarking Tool Retrieval for Large Language Models},
  author       = {Shi, Zhengliang and Wang, Yuhan and Yan, Lingyong and Ren, Pengjie and Wang, Shuaiqiang and Yin, Dawei and Ren, Zhaochun},
  booktitle    = {Findings of the Association for Computational Linguistics: ACL 2025},
  pages        = {24497--24524},
  year         = {2025},
  doi          = {10.18653/v1/2025.findings-acl.1258},
  url          = {https://aclanthology.org/2025.findings-acl.1258/}
}

@inproceedings{chen2024bgem3,
  title        = {{M3-Embedding}: Multi-Linguality, Multi-Functionality, Multi-Granularity Text Embeddings Through Self-Knowledge Distillation},
  author       = {Chen, Jianlyu and Xiao, Shitao and Zhang, Peitian and Luo, Kun and Lian, Defu and Liu, Zheng},
  booktitle    = {Findings of the Association for Computational Linguistics: ACL 2024},
  pages        = {2318--2335},
  year         = {2024},
  doi          = {10.18653/v1/2024.findings-acl.137},
  url          = {https://aclanthology.org/2024.findings-acl.137/}
}

@misc{jevk5,
  title        = {{JevK5}: An Open-Weight {Jev} Alternative},
  author       = {{JevK5 authors}},
  year         = {2026},
  howpublished = {\url{https://github.com/allebee/jevk5}},
  note         = {4B model v0.3}
}

@misc{intern_decision,
  title        = {{Intern-Decision}: Fast Multi-Modal Decision Model},
  author       = {{InternLM}},
  year         = {2026},
  howpublished = {\url{https://github.com/InternLM/Intern-Decision}},
  note         = {Code and model release}
}

@misc{clm,
  title        = {Contrastive Language Models},
  author       = {Kwok, Jacky and Kang, Hangoo and Suresh, Tarun and Saad-Falcon, Jon and Pavone, Marco and R{\'e}, Christopher and Mirhoseini, Azalia},
  year         = {2026},
  howpublished = {\url{https://github.com/Contrastive-LM/CLM}},
  note         = {Technical blog and official implementation}
}

@inproceedings{schick2023toolformer,
  title        = {{Toolformer}: Language Models Can Teach Themselves to Use Tools},
  author       = {Schick, Timo and Dwivedi-Yu, Jane and Dess{\`i}, Roberto and Raileanu, Roberta and Lomeli, Maria and Hambro, Eric and Zettlemoyer, Luke and Cancedda, Nicola and Scialom, Thomas},
  booktitle    = {Advances in Neural Information Processing Systems},
  volume       = {36},
  pages        = {68539--68551},
  year         = {2023},
  url          = {https://proceedings.neurips.cc/paper_files/paper/2023/hash/d842425e4bf79ba039352da0f658a906-Abstract-Conference.html}
}

@inproceedings{qin2024large,
  title        = {Large Language Models are Effective Text Rankers with Pairwise Ranking Prompting},
  author       = {Qin, Zhen and Jagerman, Rolf and Hui, Kai and Zhuang, Honglei and Wu, Junru and Yan, Le and Shen, Jiaming and Liu, Tianqi and Liu, Jialu and Metzler, Donald and Wang, Xuanhui and Bendersky, Michael},
  booktitle    = {Findings of the Association for Computational Linguistics: NAACL 2024},
  pages        = {1504--1518},
  year         = {2024},
  doi          = {10.18653/v1/2024.findings-naacl.97},
  url          = {https://aclanthology.org/2024.findings-naacl.97/}
}

@inproceedings{zheng2023judging,
  title        = {Judging {LLM-as-a-Judge} with {MT-Bench} and {Chatbot Arena}},
  author       = {Zheng, Lianmin and Chiang, Wei-Lin and Sheng, Ying and Zhuang, Siyuan and Wu, Zhanghao and Zhuang, Yonghao and Lin, Zi and Li, Zhuohan and Li, Dacheng and Xing, Eric P. and Zhang, Hao and Gonzalez, Joseph E. and Stoica, Ion},
  booktitle    = {Advances in Neural Information Processing Systems},
  volume       = {36},
  pages        = {46595--46623},
  year         = {2023},
  url          = {https://proceedings.neurips.cc/paper_files/paper/2023/hash/91f18a1287b398d378ef22505bf41832-Abstract-Datasets_and_Benchmarks.html}
}

@misc{almeida2026jev,
  title        = {Introducing {System One Models} \& {Jev}},
  author       = {Almeida, Diogo},
  year         = {2026},
  howpublished = {\url{https://typesafe.ai/blog/introducing-system-one-models-and-jev}},
  note         = {TypeSafe AI Blog}
}

@misc{semif,
  title        = {{SemIf}: Semantic Ifs from Open Models},
  author       = {Lee, Theodore},
  year         = {2026},
  howpublished = {\url{https://github.com/TheoLeeCJ/SemIf-OpenJev}},
  note         = {GitHub repository; formerly OpenJev}
}

@misc{reflex,
  title        = {{reflex}: A Small Open Decision Model},
  author       = {Raghavan, Kshetrajna},
  year         = {2026},
  howpublished = {\url{https://github.com/kshetrajna12/reflex}},
  note         = {Code and model release}
}

@misc{palmer2026kev4b,
  title        = {{Kev-4B}: A Calibrated Decision Model on {Qwen3.5-4B}},
  author       = {Palmer, Jared},
  year         = {2026},
  howpublished = {\url{https://huggingface.co/jaredpalmer/kev-4b}},
  note         = {Kev 1.0}
}

@misc{laya,
  title        = {{Laya}: Multilingual, Non-Autoregressive {System 1} Decision Engine},
  author       = {{Convai Innovations}},
  year         = {2026},
  howpublished = {\url{https://github.com/NandhaKishorM/laya}},
  note         = {Code and model release}
}

@misc{localllama2026typed,
  title        = {Typed Decisions},
  author       = {{LocalLLaMA}},
  year         = {2026},
  howpublished = {\url{https://huggingface.co/datasets/LocalLLaMA/typed-decisions}},
  note         = {Hugging Face dataset}
}

@inproceedings{fitzgerald2023massive,
  author       = {FitzGerald, Jack and
               Hench, Christopher and
               Peris, Charith and
               Mackie, Scott and
               Rottmann, Kay and
               Sanchez, Ana and
               Nash, Aaron and
               Urbach, Liam and
               Kakarala, Vishesh and
               Singh, Richa and
               Ranganath, Swetha and
               Crist, Laurie and
               Britan, Misha and
               Leeuwis, Wouter and
               Tur, Gokhan and
               Natarajan, Prem},
  title        = {{MASSIVE}: A 1{M}-Example Multilingual Natural Language Understanding Dataset with 51 Typologically-Diverse Languages},
  booktitle    = {Proceedings of the 61st Annual Meeting of the Association for Computational Linguistics (Volume 1: Long Papers)},
  year         = {2023},
  pages        = {4277--4302},
  publisher    = {Association for Computational Linguistics},
  doi          = {10.18653/v1/2023.acl-long.235},
  url          = {https://aclanthology.org/2023.acl-long.235/}
}
\newpage
\appendix
\section{Derivations and Structural Properties}
\label{app:jeva-foundations}

\subsection{Entropy-Regularized Candidate Reading}
\label{app:jeva-entropy}
For compatibility vector $c$ and probability simplex $\Delta_K$, the attention weights solve
\begin{equation}
 \begin{aligned}
 a_\ell^h&=\arg\max_{a\in\Delta_K}
 \left\{a^\top c+\tau_\ell\mathcal H(a)\right\},\\
 \mathcal H(a)&=-\sum_i a_i\log a_i.
 \end{aligned}
 \label{eq:jeva-entropy}
\end{equation}
This entropy-regularized selection balances compatibility against the spread of evidence. Learned value and output matrices determine how that evidence changes the state.

\paragraph{Why this aggregation rule?}
Its variational characterization is exact. Let
$Z=\sum_i\exp(c_i/\tau_\ell)$ and
$a_i^*=\exp(c_i/\tau_\ell)/Z$. For every $a\in\Delta_K$, using
$0\log0=0$,
\begin{equation}
 a^\top c+\tau_\ell\mathcal H(a)
 =\tau_\ell\log Z-\tau_\ell D_{\mathrm{KL}}(a\Vert a^*).
 \label{eq:jeva-attention-kl}
\end{equation}
Since $D_{\mathrm{KL}}(a\Vert a^*)\geq0$, the objective is at most $\tau_\ell\log Z$, attained uniquely at $a=a^*$. This proves Eq.~\eqref{eq:jeva-entropy}. The weights are therefore the unique entropy-regularized allocation of attention over candidate evidence. They are intermediate read weights, not the final decision probabilities: the state may gather evidence from several options before committing to one.

\paragraph{Block parameterization.}
Here $F_\ell(r)=W_{\ell,\mathrm{down}} [\operatorname{SiLU}(W_{\ell,\mathrm{gate}}r)\odot W_{\ell,\mathrm{up}}r]$ is a bias-free SwiGLU network with hidden width $f$. We initialize the four decision projection matrices to identity and the SwiGLU output matrix to zero. Subsequent optimization learns both branches. Each temperature is parameterized as $\tau=10^{-4}+\operatorname{softplus}(\eta)$ with its own trainable scalar $\eta$. The coupling gives $0<\alpha_\ell/\tau_\ell<1/\sqrt L$, scaling the mixing coefficient with both selection temperature and decision depth.

\paragraph{Separate temperatures for separate questions.} $\tau_\ell$ controls evidence aggregation within a block, $\tau_\mathrm{dec}$ controls the final competition, and $\tau_\mathrm{cov}$ controls the slope of the answer-availability gate. For fixed logits $c$ and $a(\tau)=\softmax(c/\tau)$, $\mathrm d\mathcal H(a(\tau))/\mathrm d\tau =\operatorname{Var}_{i\sim a(\tau)}(c_i)/\tau^3\geq0$ (Appendix~\ref{app:jeva-sensitivity}). Decoupling the parameters therefore lets training tune evidence mixing without forcing an identical readout temperature. The numerical temperatures need not be directly comparable: per-head cosine similarities and final state--option dot products have different scales.

\subsection{Geometry and Single-Block Expressivity}
\label{app:jeva-setproof}
\paragraph{Geometry of the selection correction.}
The candidate interaction can be seen directly in score space. Let $S_h$ insert a head vector into the concatenation, let $W_\ell^{V,h}$ be the corresponding value-projection rows, and define
\begin{equation}
 G_{\ell,ij}^{h}=u_i^\top W_\ell^O S_h W_\ell^{V,h}u_j.
 \label{eq:jeva-geometry}
\end{equation}
Before the channel-refinement branch, the score of candidate $i$ changes by
\begin{equation}
 u_i^\top\widetilde z_\ell-u_i^\top z_\ell
 =\alpha_\ell\sum_{h,j}a_{\ell j}^{h}G_{\ell,ij}^{h}.
 \label{eq:jeva-correction}
\end{equation}
The distinction from independent scoring already appears with one block.

We give the construction for Proposition~\ref{prop:jeva-setcontext}.
\begin{proof}
In decision representation space, use a two-dimensional, single-head block with identity Q/K/V/O projections and zero FFN output. Let
\begin{equation}
 \begin{gathered}
 z_0=(1,0),\quad A=(4,3)/5,\\
 B=(4,-3)/5,\quad C_\pm=(3,\pm4)/5,\\
 \mathcal O_\pm=(A,B,C_\pm).
 \end{gathered}
 \label{eq:jeva-context-construction}
\end{equation}
All vectors are unit length, so query/key normalization leaves them unchanged. Choose $\tau_\mathrm{set}=3/20$ and $\gamma=-\log2$, giving $\alpha=1/20$. The weight of $C_\pm$ is $t=(1+2e^{4/3})^{-1}$. The two refined states are $z_\pm=(1+\alpha(4-t)/5,\,\pm4\alpha t/5)$, hence
\begin{equation}
 s_A(\mathcal O_\pm)-s_B(\mathcal O_\pm)
 =\pm\frac{24\alpha t}{25}.
 \label{eq:jeva-context-flip}
\end{equation}
The winning $A$ or $B$ also exceeds $C_\pm$: its initial score advantage is $1/5$, and a correction of norm at most $\alpha$ can reduce a pairwise margin by at most $2\alpha<1/5$. Thus both winners are unique. An independent scorer cannot reverse this relative ordering by Eq.~\eqref{eq:jeva-independent-odds}; equal scores with a fixed tie-breaking rule also select the same one of $A,B$ in both sets. To embed the construction into $H$ heads of width at least two, place one copy of each two-dimensional vector in every head, zero-pad the remaining coordinates, and scale each copy by $H^{-1/2}$. Per-head query/key normalization recovers the original attention weights, while the summed headwise dot products reproduce the original scores.
\end{proof}

Thus, evidence selected from one candidate can affect another candidate's score through learned geometry, without updating either candidate. For a single head with identity value/output maps, this reduces to the Gram correction $\alpha_\ell UU^\top a_\ell$. The deployed multi-head model learns a more general interaction. It computes the weighted message first, never materializing a $K\times K$ matrix.

The identity in Eq.~\eqref{eq:jeva-correction} explains how fixed candidate vectors still support relational decisions. Candidate $j$ contributes to candidate $i$ through $G_{\ell,ij}^{h}$, weighted by its relevance to the current state. The matrix is an algebraic description of the interaction, not a stored tensor.

The proposition identifies a representational capability, not a guarantee that every set change should reverse a preference. It shows why one decision block is already functionally distinct from an independent similarity readout. Its state update can reflect relations within the current set while retaining cacheable candidate features. Figure~\ref{fig:jeva-context-mechanism} visualizes the construction. Conversely, low total attention on added alternatives gives an explicit bound on their influence; Appendix~\ref{app:jeva-added-candidates} derives this bound for the actual matrix-valued mixer.

\subsection{Coverage and Candidate Count}
\label{app:jeva-coverage-proof}
The following expands Proposition~\ref{prop:jeva-coverage}. Fix the final candidate scores. Adding scores no larger than the existing maximum leaves $\rho$ in Eq.~\eqref{eq:jeva-coverage} unchanged. By contrast, a flat softmax with a fixed null score $b_0$ and $K$ equal candidate scores $s$ gives
\begin{equation}
 \begin{aligned}
 p^{\mathrm{flat}}_\bot
 &=\frac{1}{1+K\exp((s-b_0)/\tau)}\\
 &=\sigma\!\left(\frac{b_0-s}{\tau}-\log K\right).
 \end{aligned}
 \label{eq:jeva-null-cardinality}
\end{equation}
\begin{proof}
The maximum, and thus the argument of the coverage sigmoid, is unchanged under the stated addition. For flat softmax, summing the $K$ identical candidate exponentials and dividing by the null exponential gives Eq.~\eqref{eq:jeva-null-cardinality}.
\end{proof}

Proposition~\ref{prop:jeva-coverage} removes a compulsory $\log K$ contribution to the \emph{coverage} decision. Set refinement can still change the scores, and competition still redistributes probability across options. In particular, a joint argmax over null and individual candidates is distinct from thresholding coverage. The factorization assigns these effects to explicit components that can be trained and examined separately. Figure~\ref{fig:jeva-coverage-mechanism} illustrates the two components with fixed-score examples.

\subsection{Detailed Reuse and Symmetry Proofs}
\label{app:jeva-structural-proofs}
\paragraph{Exact reuse.}
The following expands the reuse part of Proposition~\ref{prop:jeva-symmetry}.
\begin{proof}
Equation~\eqref{eq:jeva-memory} contains no state, question, or other candidate. A stored entry is thus the same value as a recomputed entry. Gathering entries in the requested order supplies identical inputs to every block. Induction through Eq.~\eqref{eq:jeva-update} yields identical states, followed by identical competition, coverage, and typed readouts.
\end{proof}

The result concerns reuse of \emph{computation}, not reuse of an answer. Figure~\ref{fig:jeva-reuse-mechanism} illustrates reuse across new requests whose candidate sets overlap without repeating. Even with every option cached, the request is encoded and the full decision module is evaluated. Both semantic vectors and layer-specific key/value projections can be reused: learned matrix views do not break cacheability when they depend only on an unchanged option.

With FFN width $f$, preparing layer views costs $O(LKd^2)$, while a request with prepared memories costs $O(L(d^2+Kd+df))$ for the decision module, in addition to state encoding. For $f=4d$, this is $O(L(d^2+Kd))$. A bank of $M$ entries stores $O(M(L+1)d)$ numbers. Total computation is the sum of text encoding and decision costs: encoding the request and any uncached options, preparing missing views, and applying the linear-in-$K$ candidate interaction. Reuse removes the corresponding option-side terms. Producing a dense distribution over all $K$ candidates itself requires $\Omega(K)$ output writes. At fixed width and depth, the prepared-memory head therefore has optimal order of growth in $K$ for this output interface. This is a computation bound, not an accuracy extrapolation claim. Independent encoding also assigns a separate token budget to each text, rather than requiring the entire option list to fit one concatenated context window.

\paragraph{Candidate symmetry.}
We expand the symmetry part of Proposition~\ref{prop:jeva-symmetry}.
\begin{proof}
Let $\pi$ be a permutation of $\{1,\ldots,K\}$, with candidates and their metadata reordered together, so $u'_i=u_{\pi(i)}$. Independent encoding leaves $z'_0=z_0$ and permutes the option rows. If $z'_\ell=z_\ell$, the queries agree, keys and values permute, and the softmax denominator is unchanged. Therefore $a'{}_{\ell i}^h=a_{\ell,\pi(i)}^h$, and reindexing gives
\begin{equation}
 \sum_i a'{}_{\ell i}^h v'{}_{\ell i}^h
 =\sum_j a_{\ell j}^h v_{\ell j}^h.
 \label{eq:jeva-permutation-proof}
\end{equation}
Both residual branches receive the same state and message, so $z'_{\ell+1}=z_{\ell+1}$. Induction yields $z'_L=z_L$. Final dot products and their softmax permute, whereas their maximum and coverage do not. Score values move with their candidates, so reindexing preserves the weighted expectation. Noul continues to read the same positive candidate's probability. A unique maximizing candidate retains its identity.
\end{proof}

This structural symmetry requires no permutation augmentation or consistency loss. It concerns reorderings of the same alternatives; candidate multiplicity and set composition remain meaningful inputs. Numerical tolerances and the tied-maximum case are described in Appendix~\ref{app:jeva-method}.

\subsection{Joint Cross-Entropy Identity}
\label{app:jeva-ce-identity}
The open-set loss in Eq.~\eqref{eq:jeva-loss} is joint cross-entropy.
Indeed,
\begin{equation}
 \begin{aligned}
 &-y_\bot\log p_\bot-\sum_i y_i\log p_i\\
 &\quad=-y_\bot\log(1-\rho)-c_y\log\rho\\
 &\qquad\quad-\sum_i y_i\log q_i\\
 &\quad=\operatorname{BCEWithLogits}(g,c_y)\\
 &\qquad\quad-\sum_i y_i\log q_i,
 \end{aligned}
 \label{eq:jeva-loss-proof}
\end{equation}
using $y_\bot+c_y=1$ and $\rho=\sigma(g)$. Under this primary joint log loss, all-null examples supervise coverage without assigning a relative-ranking target; answer-present examples supervise both coverage and competition. The corresponding population log-loss decomposition is given in Appendix~\ref{app:jeva-proper-loss}.

\section{Method Proofs and Implementation Details}
\label{app:jeva-method}

\subsection{The Single-Block Construction in Higher Dimensions}
\label{app:jeva-context}

Proposition~\ref{prop:jeva-setcontext} is an existence statement in decision representation space. To embed its two-dimensional vectors in $H$ heads of width $d_h\geq2$, define
\[
 J(v)=H^{-1/2}\operatorname{Concat}_{h=1}^{H}(v_1,v_2,0,\ldots,0).
\]
Then $J(v)^\top J(w)=v^\top w$. For unit $v$, each nonzero head slice has norm $H^{-1/2}$; normalizing queries and keys cancels this scale. Every head therefore uses the two-dimensional attention weights. Values retain the $H^{-1/2}$ factor, and identity output projection gives $z_1=J(z_\pm)$ after mixing. Its readout scores are exactly those of the original construction. This embeds it in $d=768,H=8$. Setting $W_{\mathrm{down}}=0$ gives zero FFN output independently of RMSNorm, completing the construction within the deployed architecture.

For completeness, if an independent scorer ties $A$ and $B$, a fixed tie-breaking rule orders them identically in both sets. Equal scores therefore cannot resolve the contradiction in the proof. A common coverage factor also cancels from $p_A/p_B$, so adding a coverage gate to an otherwise independent scorer does not remove this restriction.

\subsection{Temperature and the Sensitivity of Candidate Aggregation}
\label{app:jeva-sensitivity}

\paragraph{Entropy at fixed compatibility.}
For fixed finite logits $c_i$, write $a_i(\tau)=e^{c_i/\tau}/Z(\tau)$ and $\mu(\tau)=\sum_i a_i(\tau)c_i$. Differentiating gives
\[
 \frac{\mathrm da_i}{\mathrm d\tau}
 =\frac{a_i(\mu-c_i)}{\tau^2},\qquad
 \frac{\mathrm d\mu}{\mathrm d\tau}
 =-\frac{\operatorname{Var}_{a(\tau)}(c)}{\tau^2}.
\]
Since $\mathcal H(a)=\log Z-\mu/\tau$ and $\mathrm d\log Z/\mathrm d\tau=-\mu/\tau^2$, the two mean terms cancel, yielding
\begin{equation}
 \frac{\mathrm d\mathcal H(a(\tau))}{\mathrm d\tau}
 =\frac{\operatorname{Var}_{a(\tau)}(c)}{\tau^3}\geq0.
 \label{eq:jeva-entropy-temperature}
\end{equation}
For fixed logits, attention entropy is therefore nondecreasing in temperature.

\paragraph{Why the matrix model involves cross-covariance.}
Fix one head's candidate keys and values, and treat its normalized query $q$ as the input to $\mu_v(q)=\sum_i a_i(q)v_i$. Put $\bar k=\sum_i a_i k_i$. The softmax derivative gives $\nabla_q a_i=a_i(k_i-\bar k)/\tau$, hence
\begin{equation}
 \begin{aligned}
 D_q\mu_v
 &=\frac{1}{\tau}\left[\sum_i a_i v_i k_i^\top
       -\mu_v\bar k^\top\right]\\
 &=\frac{\operatorname{Cov}_a(v,k)}{\tau}.
 \end{aligned}
 \label{eq:jeva-cross-covariance}
\end{equation}
If $\|k_i\|\leq1$ and $\|v_i\|\leq R_v$, this derivative has operator norm at most $R_v/\tau$. Indeed, for unit vectors $r,w$, Cauchy--Schwarz bounds $|\operatorname{Cov}(r^\top v,w^\top k)|$ by the square root of the two variances, which are at most $R_v^2$ and $1$.

The learned matrix model uses different keys and values, so this cross-covariance need not be symmetric or positive semidefinite. To differentiate a complete block with respect to $z$, one must also apply the query projection and normalization derivatives, the output projection, and the RMSNorm--SwiGLU residual. The coupling $\alpha/\tau<1/\sqrt L$ bounds the scalar prefactor of the cross-covariance contribution within this derivative composition.

\subsection{A Quantitative Bound for Added Candidates}
\label{app:jeva-added-candidates}

Set dependence is useful when alternatives supply relevant evidence. For weakly attended additions, its magnitude can be bounded. Fix one head's incoming query and the original keys/values. Add $r$ candidates whose compatibility logits are at most $c_{\max}-\delta$, where $c_{\max}$ is the original maximum and $\delta\geq0$. Let $\beta$ be their total attention mass after insertion. Then
\begin{equation}
 \begin{gathered}
 \beta\leq\frac{r e^{-\delta/\tau}}{1+r e^{-\delta/\tau}},\\
 \|\mu_{\mathrm{new}}-\mu_{\mathrm{old}}\|
 \leq 2R_v\beta
 \end{gathered}
 \label{eq:jeva-added-bound}
\end{equation}
provided every old and new value has norm at most $R_v$.

\begin{proof}
Let $Z=\sum_{\mathrm{old}}e^{c_i/\tau}$ and $R=\sum_{\mathrm{new}}e^{c_i/\tau}$. Since $Z\geq e^{c_{\max}/\tau}$ and $R\leq r e^{(c_{\max}-\delta)/\tau}$, $\beta=R/(Z+R)$ obeys the first bound. The original weights are all rescaled by $1-\beta$. Writing $\mu_{\mathrm{add}}$ for the normalized mean over new values gives the exact identity $\mu_{\mathrm{new}}=(1-\beta)\mu_{\mathrm{old}} +\beta\mu_{\mathrm{add}}$. Both means lie in a radius-$R_v$ ball, which proves the second bound.
\end{proof}

For the actual multi-head message, let the corresponding masses and value-norm bounds be $\beta_h,R_h$. Concatenation and the output matrix give
\begin{equation}
 \|\Delta m\|\leq2\|W^O\|_2
 \left(\sum_h R_h^2\beta_h^2\right)^{1/2}.
 \label{eq:jeva-added-multihead}
\end{equation}
In the single-block model, the initial state is independent of the set, so its incoming query is indeed fixed across this comparison. If $F\circ\operatorname{RMSNorm}$ has Lipschitz constant $B$ on the segment between the two post-mixer states, then
\[
 \begin{gathered}
 \|\Delta z_1\|\leq(1+B)\alpha\|\Delta m\|,\\
 \left|\Delta\log(q_i/q_j)\right|
 \leq\frac{2\|\Delta z_1\|}{\tau_{\mathrm{dec}}}
 \end{gathered}
\]
for any two unchanged options. The first inequality follows by the triangle inequality through the residual FFN, and the second uses $\|u_i-u_j\|\leq2$. The bound depends on the additions' total attention mass and the local FFN sensitivity. Extending it across multiple blocks additionally requires bounds on changes in the incoming state.

\subsection{Joint Log Loss and the Meaning of Null}
\label{app:jeva-proper-loss}

Let the true conditional output distribution be $p^\star=(1-\rho^\star,\rho^\star q^\star)$ and the model distribution be $p=(1-\rho,\rho q)$. Expanding the KL divergence gives
\begin{equation}
 \begin{aligned}
 &D_{\mathrm{KL}}(p^\star\Vert p)\\
 &=(1-\rho^\star)\log\frac{1-\rho^\star}{1-\rho}
 +\rho^\star\log\frac{\rho^\star}{\rho}\\
 &\quad+\rho^\star\sum_iq_i^\star\log\frac{q_i^\star}{q_i}\\
 &=D_{\mathrm{KL}}\!\left(\operatorname{Ber}(\rho^\star)
                    \Vert\operatorname{Ber}(\rho)\right)\\
 &\quad+\rho^\star D_{\mathrm{KL}}(q^\star\Vert q).
 \end{aligned}
 \label{eq:jeva-population-kl}
\end{equation}
The excess expected primary joint log loss decomposes into coverage error and conditional competition error, with the latter weighted by $\rho^\star$. If $\rho^\star=0$, the competition term vanishes and $q^\star$ can be chosen arbitrarily.

Coverage and the joint maximum-probability decision also have distinct semantics. With $g=\operatorname{logit}(\rho)$ and $q_{\max}=\max_iq_i$, null wins strictly when
\begin{equation}
 1-\rho>\rho q_{\max}
 \quad\Longleftrightarrow\quad g<-\log q_{\max}.
 \label{eq:jeva-null-map}
\end{equation}
Under uniform competition, null wins when $g<\log K$, so the joint decision depends on both coverage and candidate count.

\subsection{Symmetry, Cache Identity, and Numerical Evaluation}

With tied maxima, the set of maximizing semantic candidates is invariant. An implementation returning the first maximizing index need not select the same tied member after permutation; identity-based tie-breaking resolves that case. The proofs assume exact arithmetic: floating-point reductions and batching changes can perturb nearly tied scores. Dropout is disabled during inference. A permutation moves complete text, masks, and output metadata; it does not rename labels within the text or rewrite the question.  Cache identity requires unchanged encoder and projection weights, tokenizer, templates, length policy, and numeric configuration. Repeated options are gathered repeatedly, preserving multiplicity. Different layers may use different K/V matrices because their outputs still depend only on the same typed option. Invalidating memories after parameter updates is necessary; old training activations are not reusable under the cache-equivalence proposition.

\subsection{Detailed Computational Cost}

For each uncached layer, the leading multiply--accumulate count is
\[
 2Kd^2+2d^2+2Kd+3df.
\]
These terms respectively represent option key/value projections, query and output projections, attention and aggregation, and SwiGLU. Prepared memories remove $2Kd^2$ from the per-request count; the final readout additionally costs $Kd$. Normalization, softmax, and memory gathering contribute linear work. No $K^2$ interaction tensor is constructed. Each atomic entry stores $d+2Ld$ scalar values, with keys evaluated in FP32. Separate assembled-set entries may add storage beyond this atomic-bank count.

\subsection{Encoding and Numerical Configuration}

The state template is \texttt{Decision type: \{t\}} followed by \texttt{Question: \{question\}} and \texttt{State:} with the state on the next line. Each independent option uses \texttt{Option for \{t\}: \{option\}}. Inputs do not use a chat template. Right padding and the last nonpadding hidden state implement pooling. Length limits are explicit configuration values: the supplied collator can truncate, while full-text evaluations must check lengths before calling it. 

The inspected Jeva-4B checkpoint has one matrix block, eight heads of width $96$, decision width $768$, and SwiGLU width $3072$. Its two $2560$-to-$768$ projections and decision core contain $13{,}370{,}117$ parameters together. Neither initial-state nor final-state normalization is enabled; the FFN uses RMSNorm with $\epsilon=10^{-6}$. Decision dropout is zero. LoRA uses rank $64$, scaling parameter $128$, and training dropout $0.05$, targeting the configured attention, gated linear-attention, and feed-forward maps. Inference loads the backbone, trained adapter, and custom head together.

Query/key normalization, compatibility scores, candidate softmax, and coverage are computed in FP32, alongside a BF16 backbone. Temperatures are learned through softplus parameterization and initialized to $\tau_{\mathrm{set}}=0.15$ and $\tau_{\mathrm{dec}}=\tau_{\mathrm{cov}}=0.10$. The mixing ratio is initialized to $\alpha/\tau_{\mathrm{set}}=0.60/\sqrt L$.

\subsection{Type-Specific Losses and Optimization}

For Score, let $\bar y_i=y_i/\max(\sum_j y_j,10^{-8})$ be the renormalized candidate target, $Q_j,Y_j$ the cumulative masses sorted by numeric value, and $R=\max(v_{\max}-v_{\min},10^{-8})$. The auxiliary components are
\[
 \left(\frac{\sum_i(q_i-\bar y_i)v_i}{R}\right)^2,
 \qquad \frac{1}{K-1}\sum_{j=1}^{K-1}(Q_j-Y_j)^2.
\]
They apply to valid numeric rubrics with $K\geq2$. An all-null target gives $\bar y=0$; enabled auxiliary terms are still evaluated in the current implementation. For Noul, with $p=q_{i^+}$ and $y=\bar y_{i^+}$, the components are $(p-y)^2$ and, for hard labels, $[m-(2y-1)\log(p/(1-p))]_+$. The latter clamps $p$ to $[10^{-6},1-10^{-6}]$ in FP32 only for computing the margin; it does not clamp the reported probability. Soft labels omit the margin.

The trainer clips the accumulated global gradient norm before synchronized optimizer updates, with configured threshold $1.0$ in the supplied recipes, and rejects nonfinite losses or gradient norms. Non-reentrant gradient checkpointing can reduce activation memory during LoRA training. In fully frozen-encoder mode the backbone instead stays in evaluation mode and runs without gradients; projections and the decision core remain trainable.

\section{Detailed Experimental Protocols}
\label{sec:appendix-experiments}

\subsection{Training curriculum}
Jeva follows a staged curriculum with a unified \textsc{Choice}, \textsc{Noul}, and \textsc{Score} interface. Training begins with basic typed decisions, intent recognition, and action selection ($\sim$80k examples), then expands to preference comparison, semantic judgments, and general reasoning ($\sim$500k). A counterfactual and verification stage combines $\sim$250k additional examples with earlier-data replay, followed by evidence-grounded and open-set decisions ($\sim$60k). Reading and logical reasoning ($\sim$20k) precede a balanced mixture of difficult examples and replay ($\sim$60k). Within this mixture, $\sim$15k examples blend original targets with teacher probabilities at a 7:3 ratio, accepting teacher guidance only when its top prediction matches the original target under two independently shuffled candidate orders; 30k replay examples equally mix original targets and preceding-model probabilities. Targeted capability refinement ($\sim$10k) follows, before training extends to long texts and large candidate pools ($\sim$80k). A feedback-driven capability-retention phase ($\sim$2k) precedes and can recur during the final scoring-focused stage with mixed-task replay ($\sim$70k). Supervision combines hard labels, soft probabilities, explicit null outcomes, and numerical rubric targets. Reported counts are per-stage training-set sizes and include replay where applicable.

\subsection{General evaluation and typed targets}

The general suite contains Typed Decisions, Banking77 and MASSIVE at candidate counts $K\in\{8,16\}$, and three Mind2Web $K=16$ slices: cross-task, cross-website, and cross-domain. Typed Decisions includes 600 Choice, 600 Noul, and 800 Score examples. Banking77 and MASSIVE include both answer-present and answer-withheld requests. Mind2Web evaluates selection from a supplied candidate pool.

General accuracy uses first-argmax correctness, including null for open-set requests, and is reported for fully supported slices. Unsupported requests are documented separately in the coverage results. Noul Brier measures squared error against the soft positive-candidate target. Score MAE and RMSE compare numerical expectations in the rubric's original units. Figure~\ref{fig:decision-quality}(a) displays Noul Brier against Score RMSE, showing Jeva's joint probability and numerical prediction fidelity within the same typed interface.

\subsection{Null operating points and candidate permutations}

The null comparison uses 3,080 matched Banking77 $K=16$ requests:
2,492 answer-present and 588 answer-absent cases. Jeva uses its coverage
head; the other interfaces use auxiliary coverage Noul outputs or an
explicit None candidate. 

Candidate-order tests use 600 Choice examples, their original order, its reversal, and up to three seeded random permutations. Outputs are aligned by candidate identity before comparison. Flip rate counts an example whenever any tested order changes its decision; all-order accuracy requires correctness under every tested order. Mean TV is total-variation distance between aligned distributions. Figure~\ref{fig:decision-quality}(c) connects each model's original and all-order accuracy. 

\subsection{Large-set support and bounded option encoding}
ToolRet contains 500 requests across six distinct-query cohorts, $K\in\{32,128,256,512,1024,2048\}$, with 83--84 queries per cohort. Supplied positives are retained; 389 requests are multi-positive. Additional candidates include weak or unjudged negatives, and accuracy accepts a top-scoring target candidate. 

Jeva completes all 500 requests and obtains 78.40\% aggregate accuracy, including 86.90\% at $K=32$ and 69.88\% at $K=2048$. JevK5 covers the complete sweep through recursive selection and obtains 64.00\% aggregate accuracy. Jeva directly scores the full candidate set in a single decision-head pass. The support map in Figure~\ref{fig:toolret-scaling-full}(c) reports every measured model--slice combination. Both evaluated InternDecision variants support 65 of 84 requests at $K=32$ (77.38\%); their remaining slices exceed the registered interface limits. The plotted support label 77 rounds this percentage. Accuracy is plotted only for full-support slices.

Request latency is measured on one NVIDIA H200 with device synchronization, excluding model loading. Each model uses its respective runtime; within-model batching and cache comparisons fix the checkpoint and query stream.

\subsection{Cache conditions, timing, and numerical agreement}

The cache ablation distinguishes an exact-set cache from a hybrid cache that also reuses individual candidate representations. Each episode runs under Off, Set, and Hybrid conditions, with cyclic mode ordering and cache reset between conditions. The completed experiment uses chunk-32/16,384-token batching and contains 171 episodes and 2,188 measured request identities. It includes different states sharing an ordered candidate bank and novel candidate sets with 0\%, 50\%, or 100\% requested option reuse. Novel-set priming uses a different question and candidate set, ensuring zero exact-set hits for the measured request. Priming is timed separately and excluded from reported request latencies. Figure~\ref{fig:candidate-reuse-full}(a) shows paired per-request speedup distributions for 46 novel-set requests at each reuse level.

\subsection{Input Templates and Typed Examples}
\label{app:jeva-input-examples}

\paragraph{Encoder inputs.}
Jeva separately formats the request and each candidate. The request template $T_x(t,\iota,x)$ is:
{\small
\begin{verbatim}
Decision type: {decision_type}
Question: {question}
State:
{state}
\end{verbatim}
}
The candidate template $T_o(t,o_i)$ is:
{\small
\begin{verbatim}
Option for {decision_type}: {option}
\end{verbatim}
}
Here, \texttt{decision\_type} is \texttt{choice}, \texttt{noul}, or \texttt{score}. Each formatted text is independently processed by the shared encoder. Candidate inputs contain neither the current request nor other candidates, and the template adds no candidate-position identifier. Consequently, changing the question, state, or candidate order does not change an otherwise identical candidate input.

\paragraph{Choice.}
Consider the state \emph{I was charged twice.}, the question \emph{Which team should handle this?}, and candidates \emph{Billing}, \emph{Technical support}, and \emph{Delivery}. The request input is:
{\small
\begin{verbatim}
Decision type: choice
Question: Which team should handle this?
State:
I was charged twice.
\end{verbatim}
}
The three candidate inputs are shown below; each line is encoded separately:
{\small
\begin{verbatim}
Option for choice: Billing
Option for choice: Technical support
Option for choice: Delivery
\end{verbatim}
}
A closed-set request selects from $q$. With \texttt{allow\_null=true}, selection instead uses $(p_\bot,p_1,p_2,p_3)$, where $p_\bot=1-\rho$ and $p_i=\rho q_i$. Jeva obtains null probability from its coverage gate; no additional \texttt{None} candidate is encoded.

\paragraph{Noul.}
For the same state, let the question be \emph{Was the customer charged twice?} and set $t=\texttt{noul}$. The default Boolean candidates are:
{\small
\begin{verbatim}
Option for noul: false
Option for noul: true
\end{verbatim}
}
The interface uses zero-based \texttt{positive\_index=1}, so the returned Noul value is $q_2$ in the paper's one-based notation. It represents the probability of the affirmative candidate, not the probability that the candidate set lacks an answer. The question is included in the request input, while the Boolean candidate representations remain reusable across questions of the same type.

\paragraph{Score.}
For the question \emph{How severe is the problem?}, set $t=\texttt{score}$ and provide the rubric:
{\small
\begin{verbatim}
Option for score: Low
Option for score: Medium
Option for score: High
\end{verbatim}
}
With numerical values $(v_1,v_2,v_3)=(0,1,2)$, the output is the expectation $\sum_i q_i v_i=q_2+2q_3$. The values specify the output scale; they are not appended as an additional encoder field. Likewise, \texttt{positive\_index} and \texttt{allow\_null} control output interpretation without entering the text templates. Each typed question produces its own request representation and candidate distribution.

\subsection{Evaluation Adapters and Cache Examples}
\label{app:jeva-adapters-cache}

\paragraph{Auxiliary coverage queries.}
For open-set evaluation, Jev, InternDecision, KEV, Re\textperiodcentered Flex, Laya, and Valen receive an auxiliary Noul question alongside candidate selection. It uses the same state and the following instruction template:
\begin{quote}
\small\ttfamily
\{question\}\par\medskip
Determine whether at least one of the following candidates correctly applies to the state:\par
- \{option\_1\}\par
- \{option\_2\}\par
\ldots\par
- \{option\_K\}
\end{quote}
The affirmative and negative criteria are:
\begin{quote}
\small\ttfamily
true: At least one supplied candidate applies.\par
false: None of the supplied candidates applies.
\end{quote}
SemiF receives the same coverage instruction as a separate two-option categorical question through its native scoring interface. These adapters introduce no new trained prediction head.

Let $\hat c$ be the resulting affirmative probability and $\hat q_i$ the conditional candidate probability. The evaluated open-set distribution is
\[
\hat p_\bot=1-\hat c,\qquad
\hat p_i=\hat c\,\hat q_i.
\]
Thus candidate selection and answer availability contribute separately to the final prediction.

\paragraph{Explicit null candidates.}
JevK5 and CLM instead receive an additional candidate appended after the original options. JevK5 uses the literal text \texttt{None}; CLM uses:
\begin{quote}
\small\ttfamily
None of the above. None of the listed options applies to this request.
\end{quote}
The appended candidate's probability is interpreted as null probability. Original candidate texts and their relative order are preserved. For evaluation, all open-set distributions are arranged as $(\hat p_\bot,\hat p_1,\ldots,\hat p_K)$ and scored by first-argmax, including null. This classification rule is separate from the threshold sweep used to measure null-detection AUROC.

\paragraph{Reuse across new states and sets.}
Consider two Choice requests using the following candidate identities:
\[
\begin{aligned}
\mathcal O^{(1)}&=(A,B,C),\\
\mathcal O^{(2)}&=(C,A,D),
\end{aligned}
\]
where $A$, $B$, $C$, and $D$ denote \emph{Billing}, \emph{Technical support}, \emph{Delivery}, and \emph{Account security}. These letters identify candidates in the example and are not inserted into their encoder inputs. The first state may report a duplicate charge, while the second reports an undelivered package.

Assume an initially empty cache, fixed model parameters, and sufficient capacity to retain the entries. The first request prepares memory for $A$, $B$, and $C$. The second request misses the exact-set cache because its ordered candidate tuple is new. Hybrid caching nevertheless retrieves the semantic vectors and layer-specific key/value projections of $C$ and $A$, computes only the missing entry for $D$, and assembles memory in the new order. Exact-set caching alone recomputes the candidate memory on this miss.

A later request with an already cached candidate tuple can reuse its assembled memory even when its state or question changes. In every case, Jeva still encodes the current request, refines its state, and computes new output probabilities. In document retrieval, different queries reuse the same long passages; in tool selection, requests draw from a shared registry of detailed tool descriptions; in assessment, different responses are evaluated against the same rubric. Caching option-side computation avoids repeatedly encoding this persistent content while preserving request-specific decisions.

\paragraph{Cache identity.}
Within a fixed model instance, individual-option keys use the formatted token sequences, and exact-set keys use their ordered tuples. Decision type is therefore part of the candidate's cache identity through $T_o$. Changes to model weights or encoder settings require cache invalidation. Proposition~\ref{prop:jeva-symmetry} establishes exact-arithmetic equivalence when cached entries equal their recomputed counterparts.

\end{document}